\documentclass[twoside,11pt]{article}

\usepackage{blindtext}

\usepackage[abbrvbib, preprint]{jmlr2e}

\usepackage{amsmath}
\usepackage{tikz}

\newtheorem{assumption}[theorem]{Assumption}

\usepackage{lastpage}

\firstpageno{1}

\begin{document}

\title{How Many Samples Are Enough for Learning Across Domains? }

\author{\name Hong Zheng \thanks{This work was completed during my Ph.D. period and was the result of a flash of inspiration. I decided to submit it to arXiv and leave its value to the readers. } \email zhinhom@icloud.com}

\maketitle

\begin{abstract}
Understanding the fundamental mechanisms of learning is essential for designing systems with strong generalization. Recent studies have shown that increasing the number of training domains, or enlarging the distribution shift among them, improves generalization when each domain contains sufficiently many data samples. However, the conditions under which the data samples can be considered sufficient remain unexplored. In this work, we fill this gap by establishing criteria for per-domain sample requirements based on the presented learning bounds. These criteria not only reveal an inverse linear scaling law between the number of training domains and the number of samples required per domain, but also explain the fundamental rationale behind the assumption of data sufficiency, thereby providing theoretical guidance for assessing the adequacy of existing datasets and constructing datasets. This differs from classical learning theory, as the number of samples required is highly dependent on the number of training domains. Additionally, we prove the close relationship between in-domain learning and out-of-domain generalization through the presented generalization bounds, and lastly discuss some key arguments. 
\end{abstract}


\section{Introduction}
\label{sec1} 
Despite the remarkable empirical success of modern machine learning (ML), theoretical understanding remains indispensable for explaining its behavior and guiding its practical applications. Beginning with the first theorem on the perceptron~\citep{b33} and culminating in the development of VC entropy and VC dimension~\citep{b17}, learning theory has established fundamental non-asymptotic learning bounds through the principle of uniform convergence and concentration inequalities. Under the independently and identically distributed (i.i.d.) data assumption, these bounds have served as fundamental analytical tools for guiding the design of robust ML models, learning systems, and datasets~\citep{b12}. However, their applicability becomes limited when the i.i.d. assumption no longer holds, such as in domain generalization (DG), where distribution shift occurs. 

Recently, several works have developed theoretical frameworks in DG by assuming that data are independently but non-identically distributed (i.n.d.) to analyze learning behaviors and guide designs. \cite{b3} and \cite{b36} show that the number of training domains significantly affects the generalization ability of trained ML models, arguing that increasing the number of training domains generally leads to better generalization on unseen domains. \cite{b23} further shows that the need for a large number of training domains essentially arises from a large degree of distribution shift among training domains, i.e., sufficient diversity of samples across the training domains. When the distribution shift among training domains is sufficiently large, good generalization to unseen domains can be achieved, even if there are few training domains. These works share a common assumption: each training domain contains a sufficient number of data samples, or even infinitely many samples. What constitutes a sufficient number of samples, however, or under what conditions this assumption holds, has not been thoroughly investigated. Namely, existing theory lacks a principled characterization of the sample requirements for learning across domains. As a result, its applicability in guiding learning and practical applications is limited. 

In this study, we address this limitation by establishing conditions under which the number of samples within each training domain is sufficient to ensure reliable learning and generalization. These conditions, importantly, reveal an inverse linear scaling law with an irreducible intercept, showing that the required number of samples per domain decreases linearly as the number of training domains increases, up to the minimal requirement. This implies that the number of samples per domain cannot be arbitrarily small, even as the number of training domains tends to infinity. It reveals the fundamental reason behind the commonly adopted assumption in previous works that each training domain should contain a sufficient number of samples. Consequently, these conditions provide theoretical guidance for dataset assessment and construction in DG. 

Our main contributions are summarized as follows: 
\begin{itemize}
\item{ We establish criteria for the per-domain sample requirements. Specifically, they are derived through the presented upper learning bounds under mild assumptions. The presented upper bounds not only characterize how the number of samples per domain affects learning but also recover previous theoretical results regarding the role of the number of training domains. 
}
\item{ We provide generalization bounds under distribution shift. These generalization bounds indicate that the tightness of the generalization error is closely related to the tightness of the learning bounds. Namely, they characterize how the generalization error behaves given a finite number of training domains whose samples satisfy the proposed criteria. This provides a theoretical justification for the relationship between in-domain learning performance and out-of-domain generalization. 
}
\end{itemize}

\section{Theoretical Results}
\label{sec3}

\subsection{Preliminares}
\label{sec3.1}

We here introduce some notation and related concepts that will be used to derive our methods. Considering ${\cal X} \subset {\mathbb R}$ and ${\cal Y} \subset {\mathbb R}$ being respectively an input and output space, we have a training set 
\begin{align}
D = \{ {\xi _1} = {(X,Y)_1}, \ldots ,{\xi _E} = {(X,Y)_E}\} \nonumber
\end{align}
of size $E$ in ${\cal D} = {\cal X} \times {\cal Y}$ drawn from unknown distributions in ${\cal P}$. Assume that each $\xi \in D$ contains $n$ data pairs $(x_l, y_l)$, i.e., $(X,Y): = \{ ({x_l},{y_l})\} _{l = 1}^n$. Under DG, $ {(X,Y)_i}{ \sim ^{i.i.d.}}{P_i}$ for all $ i \in [E]$, and $ {P_i} \ne {P_j}$ for all $i \ne j \in [E]$. Thus, the overall samples follow an independent non-identically distributed (i.n.d.) assumption. The training set $D$ is typically regarded as the source set, while the testing set $D_t$ is regarded as the target set. Similarly, ${P_i} \ne {P_t}$ for all $i \in [E]$ and all $t \in [E_t]$, where $E_t$ denotes the number of target domains in $D_t$. Furthermore, no information from $D_t$ is available during training. In practice, $n_i \ne n_j, \forall i \ne j \in [E]$; thus, without loss of generality, to simplify the analysis, we let $N = E^{-1} \sum_{i=1}^E{n_i} $ denote the average number of data samples in each training domain. 

We consider a hypothesis $ h \in {\cal H}: {\cal X} \to {\cal Y}=\{0, 1\}$ as a deterministic mapping for prediction, corresponding to a binary classification problem. Furthermore, we assume that all spaces and mappings are measurable and all sets are countable, which does not restrict the generality of the results presented hereafter. 

We define the loss function as $\gamma : {\cal H} \times {\cal D} \to {\mathbb R}_{\ge 0}$, denoting $\gamma(h,\xi)$ as the loss incurred by any hypothesis $h \in {\cal H}$ on sample $\xi \in {\cal D}$. In domain $i$, the empirical loss and the expected loss can be expressed as $\gamma _{N,i}(h):=N^{-1}\sum_{l=1}^N {\gamma(h, (x_l,y_l)_i)}$ and ${\mathbb E}_{\xi_i}[\gamma(h)]:={\mathbb E}_{\xi \sim P_i}[\gamma(h, \xi)]$, respectively. For all training domains, we denote the empirical loss by $\gamma_{N,E}=E^{-1}\sum_{i=1}^E{\gamma_{N,i}(h)}$. Note that the minimizers of $\min\gamma_{N,i}(h)$ and $\gamma_{N,E}(h)$ are different, and both are essential to our subsequent derivations. 

Given a class of measurable sets ${\cal A}$, let ${\cal H} = \{ {{\bf{1}}_A},A \in {\cal A}\} \subseteq \{ h:{\cal X} \to \{ 0,1\} \} $ denote the corresponding class of classifiers. Assuming that ${\cal A}$ is a VC-class, if $ {G_{\cal A}}(K)$ denotes the supremum of $ \# \{ A \cap {\cal X}’, A \in {\cal A}\} $ over the collection of subsets ${\cal X}’$ of ${\cal X}$ with cardinality $K$, then ${\cal A}$ has the VC property iff ${d_{vc}} = \sup \{ K:{G_{\cal A}}(K) = {2^K}\} < \infty $, where $ d_{vc}$ is called the VC-dimension of ${\cal A}$~\citep{b17}. Denote by ${P}({\cal H})$ the set of all joint distributions $P \in {\cal P}$ under domain $i$. Let $s \in {\cal H}$ denote the Bayes classifier, defined by $s(x) = {{\bf{1}}_{\eta (x) \ge 1/2}}$, where $\eta (x) = P[Y = 1|X = x]$ denotes the regression function of $Y$ given $X=x$. Consequently, both the regression function $\eta$ and the Bayes classifier $s$ depend on $P$~\citep{b15}. 

For the DG setting, we assume that $\{s_1, \ldots, s_E\}$ for $D$ and $\{s_1, \ldots, s_{E_t}\}$ for $D_t$, indicating that each domain has its own Bayes classifier. This assumption is milder than assuming the existence of a single Bayes classifier shared by all domains. One important reason is that the latter requires $\cap_{e=1}^{E_{all}} \arg\min {\mathbb E}_{\xi_e} [\gamma(h)] \ne \emptyset $, whereas assuming the existence of domain-wise optimal hypotheses $\{s_1, \ldots , s_{E_{all}} \}$ only requires $\arg\min {\mathbb E}_{\xi_e} [\gamma(h)] \ne \emptyset, \forall e \in [E_{all}]$. Here, $E_{all}$ denotes the total number of domains. In other words, the domain-wise optimal hypothesis assumption removes the need for a shared minimizer, remaining applicable even under substantial distribution shifts. 

At last, we introduce the following assumption on the hypothesis space, which is a common assumption in ML studies and ensures compatibility between measure and topology. 
\begin{assumption}
\label{assu1}
There exist some countable subset ${{\cal H}'}$ of ${\cal H}$ such that, for every $h \in {\cal H}$, there exists sequence $(h_k)$ of elements of ${{\cal H}'}$ such that, for every $\xi \in {\cal D}$, $ {\lim _{k \to \infty }}\gamma ({h_k},\xi ) = \gamma (h,\xi )$. 
\end{assumption}

\subsection{Bounds}
\label{sec3.2}
We first define the following relative loss between an assumed optimal hypothesis $s$ and any hypothesis $h \in {\cal H}$, 
\begin{align}
\ell (h, s) = \gamma (h, \cdot ) - \gamma (s, \cdot ). \nonumber
\end{align}
This relative loss $\ell$ estimates the deviation of $h$ from $s$, enabling hypothesis evaluation under arbitrary data scenarios, including label noise~\citep{b15, b3}. Specifically, $s(x)$ corresponds to the true label without label noise and to the optimal prediction otherwise, always achieving the minimal loss. 

By the considerations of the empirical loss and the expected loss, we define the corresponding relative empirical loss and expected loss as 
\begin{align}
\label{eq1}
\ell_{N,i}(h,s) = \gamma_{N,i}(h) - \gamma_{N,i}(s) 
\end{align}
and 
\begin{align}
\label{eq2}
{\mathbb E}_{\xi_i}[\ell(h ,s)] = {\mathbb E}_{\xi_i}[\gamma(h)] - {\mathbb E}_{\xi_i}[\gamma(s)], 
\end{align}
respectively, in domain $i$. 

Additionally, we consider the Tsybakov’s margin condition~\citep{b34}, defined as follows, to link the relative loss and the distance metric between hypotheses, 
\begin{align}
\label{eq3}
\ell (h, s) \ge m^{\theta}{d^{2\theta}}(h, s), \; \forall h \in {\cal H}, 
\end{align}
where $m$ is a positive constant satisfying $m < 1$. If the distribution $\eta(x)$ is well behaved around $1/2$, we consider $\theta = 1$, i.e., $\ell (h, s) \ge m{d^2}(h, s)$. 

Based on the above definitions, we begin by establishing the following learning bound. 
\begin{theorem}
\label{theo1}
Assume that $0 \le \gamma \le M$ and that ${\cal H}$ satisfies Assumption~\ref{assu1}, let ${\cal A}$ be a VC-class with dimension $d_{vc} \ge 1$. Assume that the margin condition~\eqref{eq3} holds with $m \ge \sqrt{d_{vc}/N}$. Given the training (source) set $D$, let $\hat h \in \arg\min_{h \in {\cal H}}{\gamma_{N,E}(h)}$. Then, the following inequality holds with probability at least $1-\delta$, 
\begin{align}
\label{theo1eq1}
\frac{1}{E}\sum_{i=1}^E{ \ell_{N,i}(\hat h, s_i) } \le \sqrt2 M\sqrt{ \frac{\ln(1/\delta)}{EN} } + \frac{\kappa d_{vc}}{mN}. 
\end{align}
Here, $\kappa = C(1 + \log (N{m^2}/{d_{vc}}))$, where $C$ is an absolute constant. 
\end{theorem}
\begin{proof}
Considering the random variable $Z_i = \ell_{N, i}(h, s)$ with $Z_i \in [-M, M]$, Lemma~\ref{applem0} implies that the following inequality holds with probability at least $1-\delta$, for any $h, s \in {\cal H}$, 
\begin{align} 
\frac{1}{E} \left( \sum_i^E{ \ell_{N,i}(h,s)} - {\mathbb E}_{\xi_i}[\ell(h,s)]\right) \le \sqrt2 M\sqrt{\frac{\ln(1/\delta)}{EN}}. \nonumber 
\end{align}
Let $\hat s_i \in \arg\min_{h \in {\cal H}}{\gamma_{N, i}(h)}$ and $s_i$ be the Bayes optimal classifier in domain $i$, and substitute $h$ and $s$ into the above inequality, we obtain 
\begin{align}
\frac{1}{E} \left( \sum_i^E{ \ell_{N,i}(\hat s_i,s_i)} - {\mathbb E}_{\xi_i}[\ell(\hat s_i,s_i)]\right) \le \sqrt2 M\sqrt{\frac{\ln(1/\delta)}{EN}}. \nonumber 
\end{align}
By Lemma~\ref{applem3}, it follows that 
\begin{align}
{\mathbb E}_{\xi_i}[\ell(\hat s_i,s_i)] &\le \frac{\kappa d_{vc}}{mN} \nonumber \\
\frac{1}{E}\sum_i^{E}{\mathbb E}_{\xi_i}[\ell(\hat s_i,s_i)] &\le \frac{\kappa d_{vc}}{mN} \nonumber 
\end{align}
This indicates that 
\begin{align}
\frac{1}{E} \sum_i^E{ \ell_{N,i}(\hat s_i,s_i)} \le \sqrt2 M\sqrt{\frac{\ln(1/\delta)}{EN}} + \frac{\kappa d_{vc}}{mN}. \nonumber 
\end{align}
Then, let $\hat h \in \arg\min \gamma_{N,E}(h)$, by an algebraic fact,
\begin{align}
\frac{1}{E} \sum_i^E{ \ell_{N,i}(\hat h,s_i)} \le \frac{1}{E} \sum_i^E{ \ell_{N,i}(\hat s_i,s_i)} ,\nonumber 
\end{align}
the proof is complete. Here, we should mention to readers that $\hat h$ is the minimizer for $\gamma_{N,E}(h)$, whereas $\hat s_i$ is the minimizer for $\gamma_{N, i}(h)$ for each $i \in [E]$. Namely, $\gamma_{N,E}(\hat h) \le \gamma_{N,E}(h), \; \forall h \in {\cal H}$, and let $h$ be $s_i$; then, $\gamma_{N,E}(\hat h) \le E^{-1} \sum_{i \in E} \gamma_{N, i}(s_i)$. Consequently, subtracting the same quantity from both sides preserves the inequality. 
\end{proof}
\noindent\textbf{Remarks}. (a) The left-hand side (LHS) represents the average relative empirical loss between the empirical risk minimizer $\hat h$ and all Bayes classifiers in the source domains. Note that $\hat h$ is the empirical risk minimizer across domains. (b) There are two terms in the Right-Hand Side (RHS): the first term is the main order corresponding to a convergence rate of ${\cal O}(M/\sqrt{EN})$, and the second term can be interpreted as a complexity term governed by $d_{vc}$ and $N$. (c) When $E \to \infty$ and $N \to \infty$, the upper bound becomes asymptotically tight. Given only the bounded condition $M$ for the loss function or the relative loss, the convergence rates of the upper bound cannot be improved. This is ensured by Hoeffding-type concentration inequalities and does not require an explicit derivation of corresponding lower bounds. (d) The practical implication of the above bound is to characterize how small the training loss, or the relative training loss, can be driven during training. 

The above bound clearly indicates that, as $E$ and $N$ increase, $\gamma(\hat h, \cdot) \to \gamma(s_i, \cdot), \; \forall i \in [E]$, which characterizes how $E$ and $N$ influence the prediction performance of $\hat h$ during training. Interestingly, this leads to the same conclusion as that presented in studies~\citep{b3, b36}, i.e., a large $E$ leads to better-performing trained models. Then, define $\epsilon = E^{-1}\sum_{i=1}^{E}\ell_{N,i}(\hat h, s_i) $. Assuming $\epsilon>0$, we rewrite the above bound as follows: 
\begin{align}
\epsilon \le  \sqrt{ \frac{\chi}{EN} } + \frac{\kappa d_{vc}}{mN}, \nonumber  
\end{align}
where $\chi = (2M)^2 \ln(1/\delta)$. Therefore, the following inequality for $N$ holds,  
\begin{align} 
\label{theo1eq2} 
N_1 \ge \frac{\chi}{2\epsilon^2E} + \frac{\kappa d_{vc}}{\epsilon m} . 
\end{align}
\begin{proof} 
Since 
\begin{align}
\epsilon \le \sqrt{ \frac{\chi}{EN} } + \frac{\kappa d_{vc}}{mN}, \nonumber 
\end{align}
let $P=\sqrt{\chi / E}$ and $Q=\kappa d_{vc} / m$, we then obtain 
\begin{align}
N\ge\left( \frac{1}{2\epsilon} \left( P+\sqrt{P^2 + 4\epsilon Q} \right) \right)^2, \nonumber 
\end{align}
i.e., 
\begin{align}
N \ge \left( \frac{1}{2\epsilon} \left( \sqrt{\frac{\chi}{E}} + \sqrt{\frac{\chi}{E} + 4\epsilon\frac{\kappa d_{vc}}{m}} \right) \right)^{2}.  \nonumber 
\end{align}
We then expand the square to obtain 
\begin{align}
N & \ge \frac{1}{2\epsilon^2}\left( \frac{\chi}{E} + 2\epsilon \cdot \frac{\kappa d_{vc}}{m} + \sqrt{\frac{\chi}{E}}\cdot\sqrt{\frac{\chi}{E} + 4\epsilon\frac{\kappa d_{vc}}{m}} \right) \nonumber \\
& \ge \frac{\chi}{2\epsilon^2E} + \frac{\kappa d_{vc}}{\epsilon m} + \Omega \ge \frac{\chi}{2\epsilon^2E} + \frac{\kappa d_{vc}}{\epsilon m}. \nonumber 
\end{align}
The last inequality holds since $\Omega \ge 0$. 
\end{proof}
Here, we subscript $1$ to $N$ to distinguish it from the result obtained in the next subsection, without any magical meaning. Simplifying the above result yields 
\begin{align}
N_1 \ge \chi_{\epsilon}E^{-1} + \mu_{\epsilon}, \nonumber 
\end{align}
where $\chi_{\epsilon} = \chi / 2 \epsilon^2$ and $\mu_{\epsilon} = \kappa d_{vc}/ \epsilon m$. Since $\kappa$, $\delta$, $d_{vc}$, and $m$ are typically treated as constants in classical learning theory, for a given $\epsilon$, the relationship between $N$ and $E$ follows \textbf{an inverse linear scaling law with a nonzero irreducible intercept}. In particular, when $E \to \infty$, we have $N_1 \ge \mu_{\epsilon}$. The irreducible intercept essentially limits $N$, which cannot be arbitrarily small, even if a theoretically large $E$ can be obtained. This clearly explains the essential reason for assuming sufficient data samples within each training domain in \citep{b3, b23}. Now, by Inequality~\eqref{theo1eq2}, we have a criterion for determining whether the data samples within each domain are sufficient, which can be used for assessing the sufficiency of existing data or constructing data. 

Given a training set $D$ with $E$ training domains, where $N$ satisfies the above criterion, the achievable generalization error is characterized by the following generalization bound. 
\begin{theorem}[Generalization Bound]
\label{theo2}
Under the assumptions of Theorem~\ref{theo1}, given the training (source) set $D$ and the testing (target) set $D_t$, the following inequality holds with probability at least $1-2\delta$, 
\begin{align}
\frac{1}{E_t}\sum_{t=1}^{E_t}{\mathbb E}_{\xi_t}[\ell(\hat h, s_t)] \le 2\sqrt2 M\sqrt{ \frac{\ln(1/\delta)}{EN} } + \frac{\kappa d_{vc}}{mN} + C^*\lambda^*. \nonumber 
\end{align}
Here, $\lambda^* := \sum_t\sum_i { ( d_{JS}(P_t, P_i) + d_{JS}(P_{s_t}, P_{s_i}) ) }$, $C^* = (2\sqrt{2}M) / EE_t$, $P_{s_e}:=P_{(s_e(X),Y)}$, for all $e \in [E_{all}]$, and $d_{JS}(P,Q) = \sqrt{JS(P||Q)}$, where $JS(\cdot||\cdot)$ denotes the Jensen-Shannon divergence. 
\end{theorem} 
\begin{proof}
According to Lemma~\ref{applem4}, for any $h \in {\cal H}$, any $P_i$ and $Q_j$, we have 
\begin{align}
\left| {\mathbb E}_{P_i}[\gamma] - {\mathbb E}_{Q_j}[\gamma] \right| \le 2\sqrt{2}Md_{JS}(P_i,Q_j). \nonumber 
\end{align}
Then, extending this inequality to all source and target domains, we obtain 
\begin{align}
\left| \frac{1}{E_t}\sum_{t}^{E_t}{\mathbb E}_{P_t}[\gamma] - \frac{1}{E}\sum_{i}^{E}{\mathbb E}_{P_i}[\gamma] \right| \le C^* \sum_t^{E_t}{ \sum_i^E{ d_{JS}(P_t,P_i) } },\nonumber 
\end{align}
where $C^* = (2\sqrt{2}M) / EE_t$. Let 
\begin{align}
a &= \frac{1}{E_t}\sum_t {\mathbb E}_{\xi \sim P_t}[\gamma(h,\xi)-\gamma(s_t, \xi)] \nonumber \\
b &= \frac{1}{E_t}\sum_t {\mathbb E}_{\xi \sim P_t}[\gamma(s_t, \xi)] \nonumber \\
c &= \frac{1}{E}\sum_i {\mathbb E}_{\xi \sim P_i}[\gamma(h,\xi)-\gamma(s_i, \xi)] \nonumber \\
d &= \frac{1}{E}\sum_i {\mathbb E}_{\xi \sim P_i}[\gamma(s_i, \xi)] . \nonumber 
\end{align}
Applying the above inequality and the triangle inequality yields that 
\begin{align}
|a-c| \le |d - b| + C^* \sum_t\sum_i d_{JS}(P_t,P_i). \nonumber 
\end{align}
Similarly, by Lemma~\ref{applem4}, we obtain 
\begin{align}
|d - b| = \left|  \frac{1}{E}\sum_i {\mathbb E}_{\xi \sim P_i}[\gamma(s_i, \xi)] - \frac{1}{E_t}\sum_t {\mathbb E}_{\xi \sim P_t}[\gamma(s_t, \xi)] \right| \le C^* \sum_t{ \sum_i{ d_{JS}(P_{s_t},P_{s_i}) } }. \nonumber 
\end{align}
Let $\lambda^* := \sum_t\sum_i {\left( d_{JS}(P_{s_t},P_{s_i}) + d_{JS}(P_t,P_i) \right)}$, and this follows that 
\begin{align}
|a-c| \le C^* \lambda^* , \nonumber 
\end{align}
that is
\begin{align}
\frac{1}{E_t}\sum_t {\mathbb E}_{\xi \sim P_t}[\gamma(h,\xi)-\gamma(s_t, \xi)] &\le \frac{1}{E}\sum_i {\mathbb E}_{\xi \sim P_i}[\gamma(h,\xi)-\gamma(s_i, \xi)] + C^*\lambda^* \nonumber \\
\frac{1}{E_t}\sum_t {\mathbb E}_{\xi_t}[\ell(h,s_t)] &\le \frac{1}{E}\sum_i {\mathbb E}_{\xi_i}[\ell(h, s_i)] + C^*\lambda^*.  \nonumber 
\end{align}
By Lemma~\ref{applem0}, with probability at least $1-\delta$, we obtain 
\begin{align}
\frac{1}{E}\sum_i {\mathbb E}_{\xi_i}[\ell(h, s_i)] \le \frac{1}{E}\sum_i \ell_{N, i}(h,s_i) + \sqrt2 M\sqrt{\frac{\ln(1/\delta)}{EN}}. \nonumber 
\end{align} 
This follows that 
\begin{align}
\frac{1}{E_t}\sum_t {\mathbb E}_{\xi_t}[\ell(h,s_t)] \le \frac{1}{E}\sum_i \ell_{N, i}(h,s_i) + \sqrt2 M\sqrt{\frac{\ln(1/\delta)}{EN}} + C^*\lambda^* . \nonumber 
\end{align}
Then, considering $\hat h \in \arg\min \gamma_{N,E}(h)$ and applying Theorem~\ref{theo1}, we complete the proof.
\end{proof}
\noindent\textbf{Remarks}. (a) The LHS represents the average relative expected loss between $\hat h$ and the Bayes classifiers in the target domains, i.e., the generalization error. Note that this quantity is an expected value rather than an empirical one, and that $\hat h$ is the empirical risk minimizer across domains. (b) The RHS still contains the main order term and the complexity term, along with an additional term $\lambda^*$. This term $\lambda^*$ captures two types of joint distribution discrepancies. $P_e$ represents the original joint distribution, whereas $P_{s_e}$ represents the joint distribution induced by $s$. It characterizes the essential challenge in DG, i.e., generalization is affected by the distribution discrepancy between the source and target domains. Unfortunately, there is nothing we can do about this term, as access to $D_t$ is unavailable during learning and any information about the distribution is unknown. We further discuss this term in Section~\ref{sec3.4}. (c) $d_{JS}(\cdot, \cdot)$ is a metric that satisfies symmetry and the triangle inequality. When the JS divergence is defined using the natural logarithm, $0 \le d_{JS}(\cdot, \cdot) \le \sqrt{\ln 2}$. 

Accordingly, given $\lambda^*$, $\hat h$ is theoretically guaranteed to perform well on unseen target domains as $E \to \infty$ and $N \to \infty$. Since the parameters $E$ and $N$ are determined by the source domains, the above result implies that in-domain learning governs out-of-domain generalization, which is consistent with the empirical findings in \citep{b39}, where this argument was validated without access to any information on $\lambda^*$. In Section~\ref{sec3.4}, we further exhibit how variations in $E$ and $N$ affect the generalization error through theoretical predictions.

\subsection{Tighter Bounds}
\label{sec3.3}

In the previous subsection, we mentioned that, based solely on the information provided by the bounded loss function, the convergence rate of bound~\eqref{theo1eq1} cannot be further improved. In this subsection, we assume that additional information is available for deriving learning and generalization bounds. Then, under this assumption, we investigate how the criterion for per-domain sample requirements can be established. 

We first present the following learning bound. 
\begin{theorem}
\label{theo3}
Assume that $0 \le \gamma \le M$ and that ${\cal H}$ satisfies Assumption~\ref{assu1}, let ${\cal A}$ be a VC-class with dimension $d_{vc} \ge 1$. Assume that the margin condition~\eqref{eq3} holds with $m \ge \sqrt{d_{vc}/N}$. Given the training (source) set $D$, let $\hat h \in \arg\min_{h \in {\cal H}}{\gamma_{N,E}(h)}$, $\hat s_i \in \arg\min_{h \in {\cal H}}{\gamma_{N, i}(h)}, \; \forall i \in [E]$, and 
\begin{align}
\sigma^2 = \frac{1}{E}\sum_{i=1}^E{Var(\ell_{N,i}(\hat s_i, s_i))}. \nonumber
\end{align}
Then, the following inequality holds with probability at least $1-\delta$, 
\begin{align}
\label{theo3eq1}
\frac{1}{E}\sum_{i=1}^E{ \ell_{N,i}(\hat h, s_i) } \le  \sqrt{ \frac{2\sigma^2\ln(1/\delta)}{EN} } + \frac{4M\ln(1/\delta)}{3EN} + \frac{\kappa d_{vc}}{mN}. 
\end{align}
\end{theorem}
\begin{proof}
Let $\hat s_i \in \arg\min_{h \in {\cal H}}{\gamma_{N, i}(h)}$ and $s_i$ be the Bayes optimal classifier in domain $i$. Consider the random variable $Z_i = \ell_{N,i}(\hat s_i,s_i) - {\mathbb E}_{\xi_i}[\ell(\hat s_i,s_i)]$ with $|Z_i | \le 2M $. Then, we have 
\begin{align}
\sigma^2 = \frac{1}{E}\sum_i^E{Var(Z_i)} 
= \frac{1}{E}\sum_i^E{Var(\ell_{N,i}(\hat s_i, s_i))} 
= \frac{1}{EN}\sum_i^E{\sum_l^N { Var(\ell_{l,i}(\hat s_i, s_i)) }} \nonumber 
\end{align} 
Lemma~\ref{applem1} implies that the following inequality holds with probability at least $1-\delta$, 
\begin{align}
\frac{1}{E} \left( \sum_i^E{ \ell_{N,i}(\hat s_i,s_i)} - {\mathbb E}_{\xi_i}[\ell(\hat s_i,s_i)]\right) \le \sqrt{\frac{2\sigma^2\ln(1/\delta)}{EN}} + \frac{4M\ln(1/\delta)}{3EN}. \nonumber 
\end{align}
By Lemma~\ref{applem3}, it follows that 
\begin{align}
{\mathbb E}_{\xi_i}[\ell(\hat s_i,s_i)] &\le \frac{\kappa d_{vc}}{mN} \nonumber \\
\frac{1}{E}\sum_i^{E}{\mathbb E}_{\xi_i}[\ell(\hat s_i,s_i)] &\le \frac{\kappa d_{vc}}{mN} \nonumber 
\end{align}
This indicates that 
\begin{align}
\frac{1}{E} \sum_i^E{ \ell_{N,i}(\hat s_i,s_i)} \le \sqrt{\frac{2\sigma^2\ln(1/\delta)}{EN}} + \frac{4M\ln(1/\delta)}{3EN} + \frac{\kappa d_{vc}}{mN}. \nonumber 
\end{align}
Then, let $\hat h \in \arg\min \gamma_{N,E}(h)$, by an algebraic fact,
\begin{align}
\frac{1}{E} \sum_i^E{ \ell_{N,i}(\hat h,s_i)} \le \frac{1}{E} \sum_i^E{ \ell_{N,i}(\hat s_i,s_i)} \nonumber 
\end{align}
the proof is complete. 
\end{proof}
\noindent\textbf{Remarks}. (a) The LHS is the average relative empirical loss of $\hat h$ with respect to the Bayes classifiers in the source domains. (b) The RHS consists of three terms: the main order term with a convergence rate ${\cal O}(\sigma/\sqrt{EN})$, a bias term governed by $M$ and $EN$, and a complexity term analogous to that in previous bounds. (c) Still, when $E \to \infty$ and $N \to \infty$, the upper bound becomes tight. We say that the above bound is tighter than bound~\eqref{theo1eq1} because it incorporates $\sigma^2$. This is a characteristic feature of Bernstein's concentration inequality: when $\sigma^2$ is small, the above bound is tighter than Hoeffding-type bounds. Accordingly, lower bounds for establishing the optimality of the convergence rate remain unnecessary. Here, $\sigma^2$ provides additional information beyond the boundedness of $\gamma$ by $M$, improving the convergence rate through the constant factor. However, when $\sigma^2$ is large, the tightness of the bound may deteriorate compared with bound~\eqref{theo1eq1}. (d) The practical implication of bound~\eqref{theo3eq1} is similar to that of bound~\eqref{theo1eq1}. 

We further explain the additional information $\sigma^2$. It measures the extent to which the empirical risk minimizer $\hat s_i$ approaches the predictive performance of the Bayes classifier. To achieve $\sigma^2 \to 0$, classical learning theory requires sufficiently many samples in each domain, further supporting the necessity of the sufficient-sample assumption. We then rewrite the above bound as follows, 
\begin{align}
\epsilon  \le \sqrt{ \frac{\alpha}{EN} } + \frac{\beta}{3EN} &+ \frac{\kappa d_{vc}}{mN}, \nonumber 
\end{align}
where $\alpha = 2\sigma^2\ln(1/\delta)$ and $\beta = 4M \ln(1/\delta)$. Consequently, the following inequality for $N$ holds, 
\begin{align}
\label{theo3eq2}
N_2 \ge \frac{\alpha}{2\epsilon^2E} + \frac{\beta}{3\epsilon E} + \frac{\kappa d_{vc}}{\epsilon m} 
\end{align}
\begin{proof}
Since 
\begin{align}
\epsilon \le \sqrt{ \frac{\alpha}{EN} } + \frac{\beta}{3EN} + \frac{\kappa d_{vc}}{mN}, \nonumber 
\end{align}
let $P=\sqrt{\alpha / E}$ and $Q= \beta / 3E + \kappa d_{vc} / m$, we then obtain 
\begin{align}
N\ge\left( \frac{1}{2\epsilon} \left(P+\sqrt{P^2 + 4\epsilon Q} \right) \right)^2, \nonumber 
\end{align}
i.e., 
\begin{align}
N\ge\left( \frac{1}{2\epsilon} \left( \sqrt{\frac{\alpha}{E}}+\sqrt{ \frac{\alpha}{E} + 4\epsilon(\frac{\beta}{3E}+\frac{\kappa d_{vc}}{m}) } \right) \right)^2. \nonumber 
\end{align}
Similarly, we square the above inequality and obtain 
\begin{align}
N & \ge \frac{1}{2\epsilon^2} \left( \frac{\alpha}{E} + 2\epsilon(\frac{\beta}{3E}+\frac{\kappa d_{vc}}{m}) + \sqrt{\frac{\alpha}{E}}\sqrt{ \frac{\alpha}{E} + 4\epsilon(\frac{\beta}{3E}+\frac{\kappa d_{vc}}{m}) }  \right) \nonumber \\
&\ge \frac{\alpha}{2\epsilon^2E} + \frac{\beta}{3\epsilon E} + \frac{\kappa d_{vc}}{\epsilon m} + \Omega \ge \frac{\alpha}{2\epsilon^2E} + \frac{\beta}{3\epsilon E} + \frac{\kappa d_{vc}}{\epsilon m}.  \nonumber 
\end{align}
The last inequality holds since $\Omega \ge 0$.
\end{proof} 
Here, the subscript $2$ is used to distinguish this result from that in Inequality~\eqref{theo1eq2}. Simplifying the above result yields 
\begin{align}
N_2 \ge \alpha_{\epsilon} E^{-1} + \mu_{\epsilon}, \nonumber 
\end{align}
where $\alpha_{\epsilon} = \alpha/2\epsilon^2 + \beta/3\epsilon$. According to this inequality, the inverse relationship between $E$ and $N$ still holds, and the intercept still exists. The above lower bound for $N$ is also tighter than the lower bound~\eqref{theo1eq2}, and we will further discuss this in Section~\ref{sec3.4}. At this point, when $\sigma^2$ is available, we have a tighter criterion for determining whether the data samples within each domain are sufficient. This makes the estimation of $N$ more precise. 

Accordingly, the achievable generalization error is characterized by the following generalization bound. 
\begin{theorem}[Generalization Bound]
\label{theo4}
Under the assumptions of Theorem~\ref{theo3}, let 
\begin{align}
\hat \sigma^2 = \frac{1}{E}\sum_{i=1}^E{Var(\ell_{N,i}(\hat h, s_i))}. \nonumber 
\end{align}
Given the training (source) set $D$ and the testing (target) set $D_t$, the following inequality holds with probability at least $1-2\delta$, 
\begin{align}
\frac{1}{E_t}\sum_{t=1}^{E_t}{\mathbb E}_{\xi_t}[\ell(\hat h, s_t)] \le  \sqrt{ \frac{2\sigma^2\ln(1/\delta)}{EN} } + \sqrt{ \frac{2{\hat \sigma^2}\ln(1/\delta)}{EN} }
+ \frac{8M\ln(1/\delta)}{3EN} + \frac{\kappa d_{vc}}{mN}+ C^*\lambda^*. \nonumber 
\end{align}
\end{theorem} 
\begin{proof}
The proof is similar to that of Theorem~\ref{theo2}, and follows by applying Lemma~\ref{applem4} to obtain the following result 
\begin{align}
\frac{1}{E_t}\sum_t {\mathbb E}_{\xi_t}[\ell(h,s_t)] \le \frac{1}{E}\sum_i {\mathbb E}_{\xi_i}[\ell(h, s_i)] + C^*\lambda^*. \nonumber 
\end{align}
Let $\hat h \in \arg\min \gamma_{N,E}(h)$. Substituting $h$ with $\hat h$, we obtain 
\begin{align}
\frac{1}{E_t}\sum_t {\mathbb E}_{\xi_t}[\ell(\hat h, s_t)] \le \frac{1}{E}\sum_i {\mathbb E}_{\xi_i}[\ell(\hat h, s_i)] + C^*\lambda^*. \nonumber 
\end{align}
Considering the random variable $Z_i = \ell_{N,i}(\hat h, s_i) - {\mathbb E}_{\xi_i}[\ell(\hat h, s_i)]$ with $|Z_i | \le 2M $, we have 
\begin{align}
\hat \sigma^2 = \frac{1}{E}\sum_i{Var(\ell_{N,i}(\hat h, s_i))} \nonumber 
\end{align} 
By Lemma~\ref{applem1}, with probability at least $1-\delta$, we obtain 
\begin{align}
\frac{1}{E}\sum_i{\mathbb E}_{\xi_i}[\ell(\hat h,s_i)] \le \frac{1}{E}\sum_i{ \ell_{N,i}(\hat h,s_i)} + \sqrt{\frac{2{\hat \sigma^2}\ln(1/\delta)}{EN}} + \frac{4M\ln(1/\delta)}{3EN}. \nonumber 
\end{align} 
This follows that 
\begin{align}
\frac{1}{E_t}\sum_t {\mathbb E}_{\xi_t}[\ell(\hat h, s_t)] \le \frac{1}{E}\sum_i{ \ell_{N,i}(\hat h,s_i)} + \sqrt{\frac{2{\hat \sigma^2}\ln(1/\delta)}{EN}} + \frac{4M\ln(1/\delta)}{3EN} + C^*\lambda^*. \nonumber 
\end{align}
Then, by applying Theorem~\ref{theo3}, the proof is completed. 
\end{proof} 
\noindent\textbf{Remarks}. (a) The LHS represents the average relative expected loss between $\hat h$ and the Bayes classifiers in the target domains. (b) By the introduced $\hat \sigma^2$, the main order term has a convergence rate of ${\cal O}((\hat\sigma + \sigma)/\sqrt{EN})$, while the bias and complexity terms on the RHS are similar to those in bound~\eqref{theo3eq1}. (c) If both $\hat \sigma^2$ and $\sigma^2$ are sufficiently small, the above generalization bound is tighter than that in Theorem~\ref{theo2} for a given $\lambda^*$. 

Based on the above bound, given $\lambda^*$, the generalization ability of $\hat h$ is theoretically guaranteed, provided that $E$ and $N$ tend to infinity, which is consistent with the conclusion of Theorem~\ref{theo2}. Interestingly, $\hat \sigma^2$ is, in fact, the variance of $\epsilon$. Namely, the generalization ability of $\hat h$ is still determined by learning, which further validates the relationship between in-domain learning and out-of-domain generalization. 

\subsection{Discussion}
\label{sec3.4}

\noindent\textbf{(1) On the choice of $N$}. At the beginning, we define $N = E^{-1}\sum_{i=1}^E{n_i}$ to simplify the analysis. We then present the lower bounds for $N$ and interpret these bounds as criteria that should be satisfied by each individual training domain. The purpose is to provide a theoretical guarantee for the sufficient sample assumption in each training domain. If one does not accept the use of the average sample size $N$, one can replace $N$ with $E^{-1}\sum_{i=1}^E{n_i}$ in these inequalities, for example, Inequality~\eqref{theo1eq2}, to obtain 
\begin{align}
\sum\nolimits_{i=1}^E{n_i} \ge \chi_{\epsilon} + E \mu_{\epsilon}. \nonumber 
\end{align}
Then, by letting $N' = \sum_{i \in [E \setminus j]} n_i$, the following inequality can be used to assess whether the data samples in any domain $j \in [E]$ are sufficient: 
\begin{align}
n_j \ge \chi_{\epsilon} + E \mu_{\epsilon} - N’. \nonumber 
\end{align}

\noindent\textbf{(2) Tighter bounds}. Previously, we stated that, when only the boundedness of the loss function is available, it is no longer possible to uniformly improve the convergence rate of the learning bound at the exponential scale. Then, by introducing $\sigma^2$, the convergence rate is improved at the level of the constant factor. This leads to a tighter lower bound on $N$, i.e., a more precise criterion for per-domain sample requirements. The advantage of this criterion can be illustrated by the following example. 

Considering $\epsilon=0.05$ and $\epsilon=0.1$, by setting $\chi = (2 \times 1)^2 \times \ln(1/0.01) $, $m=0.2$, $\kappa d_{vc}=1$, $\alpha = 2 \times 1 \times \ln(1/0.01) $, and $\beta = 4 \times \ln(1/0.01)$, we plot the bound of $N_1$ and $N_2$ with respect to $E$, as shown in Figure~\ref{fig1}. As shown in this figure, the estimation of $N$ based on the tighter lower bound is more precise. Note that a smaller $\epsilon$ requires a larger value of $N$ and leads to a larger $N_{min}$. 

\begin{figure}[t]
\newcommand{\size}{0.40}
\centering
\includegraphics[width=\size\linewidth]{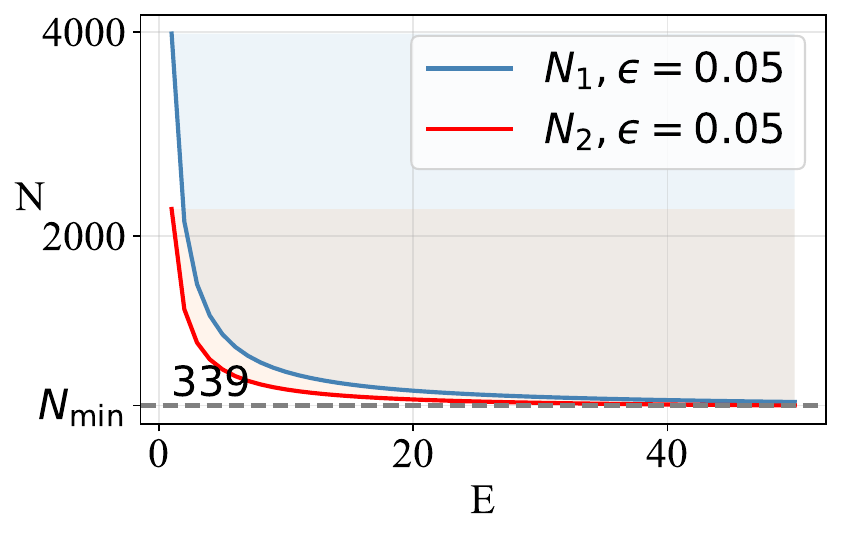}
\includegraphics[width=\size\linewidth]{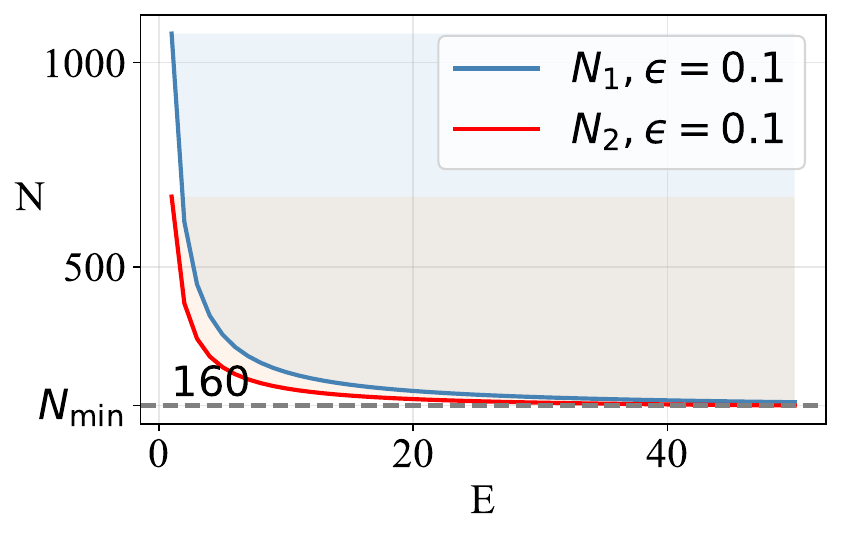}
\caption{ Examples illustrating the advantage of pursuing tighter upper bounds. The solid lines corresponding to $N_1$ and $N_2$ are obtained from our criteria by considering the equality case, while the shaded regions denote the feasible regions defined by the inequality case. $N_{min}$ denotes the limiting value of both $N_1$ and $N_2$ as $E \to \infty$. As illustrated, given $\epsilon$, $N_2$ provides a more precise estimation of $N$. Eventually, as $E$ increases, both $N_1$ and $N_2$ approach $N_{min}$.   }
\label{fig1}
\end{figure}

\noindent\textbf{(3) Recent Works}. Here, we discuss several recent works on learning and generalization bounds, while omitting their specific formulations. Interested readers are referred to the corresponding references for details. 

The lower bound in~\citep{b3} highlights the importance of $E$ in DG, showing that increasing $E$ tightens the lower bound on the minmax risk. A similar conclusion was obtained in~\citep{b36} through an upper-bound analysis. Both analyses implicitly assume sufficiently large $N$ (or even $N=\infty$), leaving the role of $N$ unexplored and limiting their applicability to practical scenarios with limited samples. 

The upper bounds in~\citep{b23} show that, under certain conditions, a sufficiently large distribution shift among training domains enables the learned model to approach the invariant prediction model~\citep{b21}. Their results suggest that, when $N$ is sufficient, increasing $E$ mainly aims to introduce larger distribution shifts across source domains. However, their analysis relies on the Kullback-Leibler divergence, which implicitly assumes absolute continuity between domain distributions, a strong assumption in DG. 

An upper bound in~\citep{b38} characterizes the effects of $E$ and $N$ on learning and reveals their inverse relationship. However, it requires a sufficiently small feature dimension; otherwise, the effects of $E$ and $N$ become negligible, limiting its applicability in determining $N$. Moreover, their impacts on generalization are not discussed. 

\textit{This work does not aim to surpass existing bounds. If one insists on such a comparison, our bounds achieve a faster convergence rate of ${\cal O}(1/\sqrt{EN})$}. 

Finally, we discuss the following learning bound, 
\begin{theorem} 
\label{theo5}
Assume that $0 \le \gamma \le M$. Given the training (source) set $D$, the following inequality holds with probability at least $1-\delta$, for any $h \in \mathcal{H}$, 
\begin{align}
\frac{1}{E}\sum_{i=1}^E{ \left( \gamma_{N,i}(h) - {\mathbb E}_{\xi_i}[\gamma(h)] \right) } \le M\sqrt{ \frac{\ln(1/\delta)}{2EN} } \nonumber 
\end{align}
\end{theorem}
\begin{proof} 
Considering the random variable $Z_i = \gamma_{N, i}(h)$ with $Z_i \in [0, M]$, Lemma~\ref{applem0} implies that the following inequality holds with probability at least $1-\delta$, for any $h \in {\cal H}$, 
\begin{align} 
\frac{1}{E}\sum_i^E{ \left( \gamma_{N,i}(h) - {\mathbb E}_{\xi_i}[\gamma(h)] \right) } \le M\sqrt{ \frac{\ln(1/\delta)}{2EN} }, \nonumber 
\end{align}
which completes the proof. 
\end{proof} 
\noindent\textbf{Remarks}. This bound is a standard learning consistency bound in DG. As $EN \to \infty$, the empirical learning process becomes increasingly consistent with the expected risk. 

This bound, however, violates the assumption of sufficient samples within each domain, as it implies $N \to 1$ when $E \to \infty$. Consequently, despite ensuring learning consistency, this bound provides limited guidance for establishing criteria and practical learning problems. 

\noindent\textbf{(4) $\lambda^*$}. A similar term was first introduced in~\citep{b14} as $\lambda = {{\epsilon}_S}({h^* }) + {{\epsilon}_T}({h^* })$, where ${h^ * } = \arg {\min _{h \in {\cal H}}}({{\epsilon}_S}(h) + {{\epsilon}_T}(h))$. When $\lambda$ is sufficiently small,  the adaptability of the learned classifier between the source domain $S$ and the target domain $T$ can be characterized by the ${\cal H}$-divergence; and vice versa. Inspired by theirs, several studies have also introduced related terms to characterize the challenge of learning across domains, such as $\hat \Delta$ in~\citep{b37}. 

In real-world applications, $\lambda^*$ is often modeled by assuming that target-domain data differ from, yet are related to, source-domain data. Explicitly evaluating it requires knowledge of the underlying distributions, which is generally impractical. 

\noindent\textbf{(5) Limitations}. A limitation of our results is the bounded-loss assumption, i.e., $0 \le \gamma \le M$, and the requirement of $\sigma^2$ and $\hat \sigma^2$ for the bounds~\eqref{theo3eq1} and related generalization bounds. Without these assumptions, the bounds cannot effectively guide DG learning or determine sample sufficiency. Moreover, assuming ${\cal Y} = \{0,1\}$ limits their direct application to regression tasks, unless $Y$ is transformed into $[0,1]$. 

\noindent\textbf{(6) Theoretical Predictions}. We define $\varepsilon = {E_t}^{-1}\sum_{t=1}^{E_t}{ {\mathbb E}_{\xi_t}[\ell(\hat h, s_t)] }$ and consider the equality case of the generalization bound. This yields the following result: 
\begin{align} 
\varepsilon = \sqrt{ \frac{\chi_2}{EN} } + \frac{\kappa d_{vc}}{mN} + C^*\lambda^*, \nonumber  
\end{align}
where $\chi_2 = (4M)^2\ln(1/\delta)$. The parameters in the above equality are chosen as follows and are not of particular significance: $\chi_2 = (4 \cdot M) \cdot \ln(1/\delta)$, $M = 1.0$, $\delta = 0.01$, $\kappa = 1.0$, $d_{vc} = 3$, $m = 0.2$, and $ C^*\lambda^* = \sqrt{\ln(2)}/2$. 

Based on these parameters, we present the theoretical predictions in Figure~\ref{fig2} by varying $E$ and $N$ separately: Group A (the left subfigure) : fixing $E$ and varying $N$; Group B (the right one): fixing $N$ and varying $E$. 

\begin{figure}[t]
\newcommand{\size}{0.40}
\centering
\includegraphics[width=\size\linewidth]{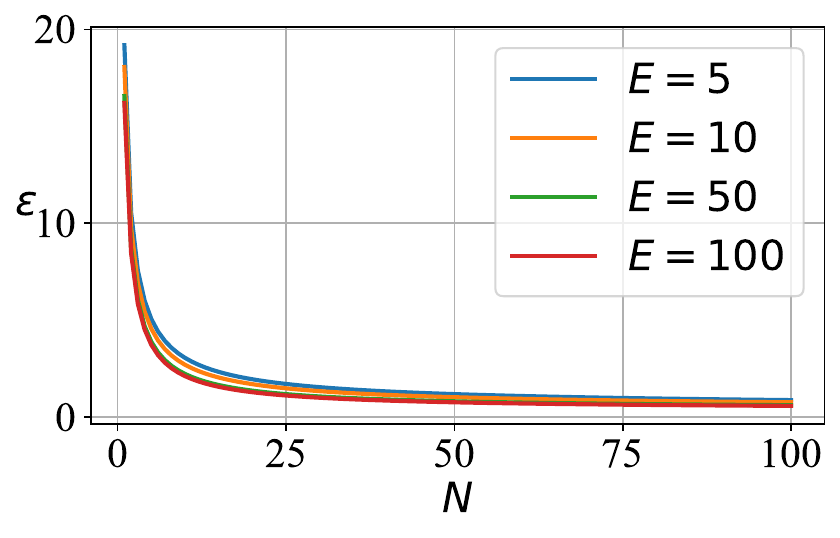}
\includegraphics[width=\size\linewidth]{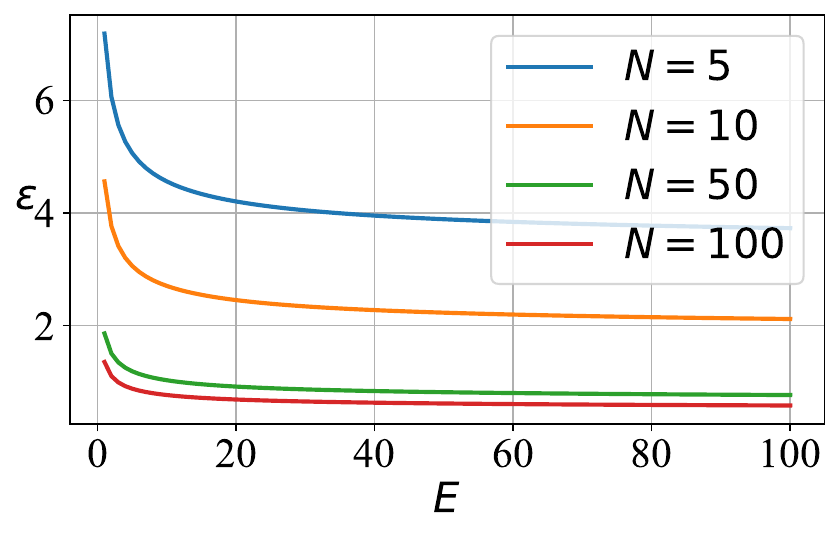}
\caption{Theoretical predictions for $\varepsilon$ by controlling $E$ and $N$. The left subfigure shows the results of $\varepsilon(N)$ with fixed $E$, while the right subfigure shows the results of $\varepsilon(E)$ with fixed $N$. The left subfigure demonstrates that, for different values of $E$, $\varepsilon$ eventually converges to similar asymptotic values as $N$ increases. The right subfigure demonstrates that, however, for different values of $N$, $\varepsilon$ converges to different asymptotic values as $E$ increases.  }
\label{fig2}
\end{figure}

As exhibited in the figure, we observe that the prediction results substantially differ across different groups. This can be explained by our criterion~\eqref{theo1eq2}. Based on the selected values of the above parameters, if we aim to achieve $\varepsilon \to 1.0$, i.e., $\epsilon \to 0.58$, the intercept term requires at least $ N \ge 25$. In Group A, all results  satisfy this condition regardless of $E$, and after increasing $N$ to $25$, $\varepsilon$ decreases to around $1.0$. However, in Group B, only the results with $N=50$ and $N=100$ satisfy this condition.

\noindent\textbf{(7) Applications}. If the assumptions of Inequalities~\eqref{theo1eq2} and \eqref{theo3eq2} hold in practical applications, then, for a desired $\epsilon$, these criteria can be applied as follows. (1) When the datasets are given, they can be used to determine whether the available samples are sufficient for ERM. One particular case is when $N_{max} < \chi_{\epsilon}E^{-1} + \mu_{\epsilon}$, where $N_{max} = \max\{n_i, \ldots, n_E\}$, corresponding to the small-sample learning problem~\citep{b12}. In this case, Structural Risk Minimization may be preferred over ERM or ERM-based learning objectives. (2) When the datasets are constructed by researchers, they can serve as guidelines for data collection.

\section{Related Works}
\label{sec2}
We briefly review several theoretical works and representative learning methods in DG. For more details, we refer the reader to the cited references. 

\noindent\textbf{Theoretical Bounds}: From the VC bounds~\citep{b12} to the Massart-type bounds~\citep{b15}, the principles of machine learning have been progressively established. These theories provide foundations for analyzing learning behaviors and guiding applications~\citep{b13}. Early work~\citep{b14} on learning across domains inspired subsequent studies, particularly in representation learning. More recently, several works have investigated the role of training domains. For example, \cite{b3,b36} study how many domains should be used for training, while~\cite{b23} shows that the degree of distribution shift among source domains affects generalization. Theoretical bounds have also been developed for specific learning methods~\citep{b37,b38}. Together, these works have laid the foundation for DG. 

\noindent\textbf{Learning Methods}: The concept of DG was first presented by Muandet~\citeyear{b24}, and has since attracted extensive research. A representative approach, invariant risk minimization~\citep{b1}, inspired methods such as IRM-Games~\citep{b25} and P-IRM~\citep{b2}, which assume that invariant features learned from source domains generalize to out-of-distribution domains. However, their effectiveness has been questioned, and failure conditions have been analyzed~\citep{b22,b27}. Causality-based methods~\citep{b28,b29} have also been proposed but require stronger assumptions. More recently, regularization-based methods~\citep{b31,b32} have shown promise for improving DG and facilitating its practical application.

\section{Conclusion}
\label{sec5} 
In this work, based on the presented criteria, we answer the question of how many samples are sufficient for learning across domains. In fact, we consider the sample requirement for a single domain, i.e., how many samples an individual sub-domain should contain in the DG setting. This result differs from classical results, such as VC bounds, in that it is linked to the number of training domains; that is, the sample requirement varies with the number of training domains. We argue that this result also provides a guarantee for single-domain generalization. However, we do not provide experiments to validate the theoretical predictions, which constitutes a limitation of this work. We leave such validation to interested readers. As an aside, with the development of large models, we are concerned about researchers' interest in the fundamental mechanisms of ML.

\section*{Acknowledgments}  
I would like to express my sincere gratitude to my supervisors, Prof. T \& Prof. L, for providing me with the time, freedom, and support to pursue my own research interests and ideas throughout my studies.

\appendix

\section{}
\label{app}

In this appendix, we provide only the lemmas, along with their proofs, used to prove our main theorems. 

\subsection{Lemmas}
\label{applemmas}

\begin{lemma}[Hoeffding Inequality under i.n.d. Settings] 
\label{applem0}
Let $Z_1,\ldots, Z_n$ be n independent random variables with $Z_i \in [a_i, b_i]$ for each $i \in [n]$. Define $S_n=\sum\nolimits_{i=1}^{n}{Z_i}$. Then for any $\varepsilon > 0$, 
\begin{align}
P\left\{ S_n - {\mathbb E}[S_n]\ge \varepsilon \right\}\le \exp\left\{-\frac{2\varepsilon^2}{\sum\nolimits_{i=1}^{n}(b_i-a_i)^2}\right\}. \nonumber 
\end{align}
\end{lemma}
\begin{proof}
This is a classical inequality, and thus we omit the proof here. For more details, please refer to~\citep{b16}. For clarity on how this lemma is applied in our case, we provide the following simple example. Consider $Z_i=\gamma(h, \xi_i)$ as the $i$-th random variable and $S_n=\sum\nolimits_{i=1}^{n}{Z_i}$, and assume that $Z_i \in [0, M_i]$. Then, based on this inequality, we have 
\begin{align}
P\left\{ S_n - {\mathbb E}[S_n]\ge \varepsilon \right\}\le exp \left\{-\frac{2\varepsilon^2}{\sum\nolimits_{i=1}^{n}M_i^2}\right\}.  \nonumber 
\end{align}
If we consider $Z_i=\ell_{N,i}(h, s)$ and assume $Z_i \in [-M_i, M_i]$, the above result still holds with different coefficients. 
\end{proof}

\begin{lemma}[Bernstein’s Inequality]
\label{applem1}
Let $Z_1, \ldots, Z_n$ be independent real-valued random variables with zero mean, and assume that $\left| {{Z_i}} \right| \le b, \; \forall i \in [n]$, with probability one. Let
\begin{align}
{\sigma ^2} = \frac{1}{n}\sum\limits_{i = 1}^n {Var\{ {Z_i}\} } . \nonumber
\end{align}
Then for any $\varepsilon > 0$, 
\begin{align}
P\left\{ {\frac{1}{n}\sum\limits_{i = 1}^n {{Z_i}} \ge \varepsilon } \right\} \le \exp \left\{ {\frac{{ - n{\varepsilon ^2}}}{{2({\sigma ^2} + \varepsilon b/3)}}} \right\}. \nonumber 
\end{align}
\end{lemma}
\begin{proof}
This is a classical inequality, and thus we omit the proof here. For more details, please refer to~\citep{b16}. We only illustrate how it is applied in our cases. We consider $Z_i= \gamma(h, \xi_i) - {\mathbb E}_{\xi_i}[\gamma(h, \xi_i)] $ as the $i$-th random variable and assume that $|Z_i| \le M$. Then, based on this inequality, we have 
\begin{align}
P\left\{ \frac{1}{n}\sum_{i = 1}^n {\gamma(h, \xi_i) - {\mathbb E}_{\xi_i}[\gamma(h, \xi_i)]} \ge \varepsilon \right\} \le \exp \left\{ \frac{{ - n{\varepsilon ^2}}}{{2({\sigma ^2} + \varepsilon M/3)}} \right\}. \nonumber 
\end{align}
Here, $\sigma = Var(\gamma(h, \xi_i))$. If we consider $Z_i=\ell_{N,i}(h, s)$ and assume $|Z_i| \le 2M$, the above result still holds with different coefficients. 
\end{proof}

\noindent\textbf{Remarks}. For the above inequalities, we consider only the one-sided formulation. This means that the corresponding high-probability inequalities are associated with $\ln(1/\delta)$, regardless of whether the left-hand side is in the form of empirical minus expected or expected minus empirical. If one considers the two-sided formulation, the corresponding high-probability inequalities are associated with $\ln(2/\delta)$, again regardless of the direction of the deviation. In either case, our main results still hold, with the only difference being a change in the coefficient within $\ln(\cdot)$.

\begin{definition}
\label{appdef1}
Let ${{\cal C}_1}$ be the class of nondecreasing and continuous functions $\psi $ from ${{\mathbb R}_{ \ge 0}}$ to ${{\mathbb R}_{ \ge 0}}$ such that $x \to \psi (x)/x$ is nonincreasing on $(0, + \infty )$ and $\psi (1) \ge 1$. 
\end{definition}

\begin{definition}
\label{appdef2}
Let $\ell (s,S) = {\inf _{h \in {\cal H}}}\ell (s,h)$ be a bias term. Given some nonnegative number $\rho$, consider a $\rho$-empirical risk minimizer, that is, any estimator $\hat s$ taking values in ${\cal H}$ such that ${\gamma _N}(\hat s) \le \rho + {\inf _{h \in {\cal H}}}{\gamma _N}(h)$. 
\end{definition}

\begin{lemma}[The Main Theorem in~\citep{b15}]
\label{applem2}
Let $ \gamma :{\cal H} \times {\cal D} \to {{\mathbb R}_{ \ge 0}}$ be a loss function such that $s$ minimizes ${\mathbb E}[\gamma (h, \cdot )]$ when $h$ varies in ${\cal H}$. Let $ {\gamma _N}(h) = (1/N)\sum\nolimits_{l \in [N]} {\gamma (h,({x_l},{y_l}))} $, $ {{\mathbb E}_\xi }[\gamma (h)] = {{\mathbb E}_\xi }[\gamma (h,\xi )]$, and $ {{\bar \gamma }_N}(h) = {\gamma _N}(h) - {{\mathbb E}_\xi }[\gamma (h)]$ and consider $d$ satisfying $Va{r_P}[\gamma (h, \cdot ) - \gamma (s, \cdot )] \le {d^2}(h,s),\forall h \in {\cal H}$. Let $ \phi $ and $w$ belong the class of functions ${\cal C}_1$ defined above and let ${\cal H}$ satisfying the separability Assumption~\ref{assu1}. Assume that, on the one hand 
\begin{align}
d(s,h) \le w(\sqrt {\ell (s,h)} ),\forall h \in {\cal H},  \nonumber 
\end{align}
Here, $d$ represents some pseudo-distance on ${\cal H} \times {\cal H}$. And that, on the other hand, one has, for every $u \in {\cal H}'$,
\begin{align}
\label{applem2eq2}
\sqrt N {\mathbb E}\left[ {\mathop {\sup }\limits_{h \in {\cal H}',d(u,h) \le \sigma } \left[ {{{\bar \gamma }_N}(u) - {{\bar \gamma }_N}(h)} \right]} \right] \le \phi (\sigma )
\end{align}
for every positive $\sigma $ such that $\phi (\sigma ) \le \sqrt N {\sigma ^2}$. Let ${\varepsilon _ * }$ be the unique positive solution of the equation 
\begin{align}
\label{applem2eq3}
\sqrt N \varepsilon _*^2 = \phi (w({\varepsilon _ * })).
\end{align}
Then, there exists an absolute constant $\kappa$ such that, for every $v \ge 1$, the following inequality holds:
\begin{align}
P\left[ {\ell (s,\hat s) > 2\rho  + 2\ell (s,S) + \kappa v\varepsilon _ * ^2} \right] \le {{\mathop{\rm e}\nolimits} ^{ - v}}. \nonumber 
\end{align}
Here, $\hat s \in \arg\min_{h \in {\cal H}} \gamma_{N}(h) $. In particular, the following risk bound is available:
\begin{align}
{\mathbb E}[\ell (s,\hat s)] \le 2(\rho  + \ell (s,S) + \kappa \varepsilon _ * ^2). \nonumber 
\end{align}
\end{lemma} 
\begin{proof}
Since this lemma is the main theorem in~\citep{b15}, the proof is omitted here. For details, please refer to page 2347 and see the proof of Theorem 2. 
\end{proof}

\begin{lemma}
\label{applem3}
Assume that ${\cal H}$ satisfies Assumption~\ref{assu1}, and let ${\cal A}$ be a VC-class with dimension $d_{vc} \ge 1$. Assume that the margin condition~\eqref{eq3} holds with $m \ge \sqrt{{d_{vc}}/N}$. There exists an absolute constant $C$ such that, if $\hat s$ denotes an empirical risk minimizer over ${\cal H}$, the following inequality holds: 
\begin{align}
{\mathbb E}\left[ {\ell (s,\hat s)} \right] \le C\left( {\frac{{{d_{vc}}(1 + \log (N{m^2}/{d_{vc}}))}}{{mN}}} \right).  \nonumber 
\end{align}
\end{lemma}
\noindent\textbf{Remarks}. (a) Considering the borderline case $m = \sqrt{{d_{vc}}/N}$, the above inequality is consistent with the well-known VC bound, i.e., ${\mathbb E}[\ell (s,\hat s)] \le C\sqrt {{d_{vc}}/N} $. Note that the initial bounds on the expected risk for a VC class established in~\citep{b12} included an additional logarithmic factor, whereas this factor can be eliminated (see~\citep{b19}) by employing chaining techniques and the concept of universal entropy. (b) This version of the lemma is a simplified one, as we only consider a bound related to the VC-dimension. For more details, please refer to the original work. (c) Here, $ \hat s = \arg {\min _{h \in {\cal H}}}\left( {{\gamma _N}(h)} \right)$.  

\begin{proof}
In fact, the proof of this lemma is provided in~\citep{b15}, so we can omit it here. However, this lemma is important to our results; therefore, we reproduce the proof process, and some important details are given on page 2339 and in the Appendix of~\citep{b15}. 

In order to apply Lemma~\ref{applem2} to the classification setting, assuming that the Bayes classifier is the target to be estimated, so that $s(x) = {{\bf{1}}_{\eta (x) \ge 1/2}}$, where $\eta (x) = P[Y = 1|X = x]$. Then, we take $d$ to be the ${L_2}(\mu )$-distance and ${\cal H} = \{ {{\bf{1}}_A},A \in {\cal A}\} $, where ${\cal A}$ is a VC-class. 

Then, the main task is to compute the moduli of continuity $\phi$ and $w$. The condition~\eqref{eq3}, i.e., $\ell (s,h) \ge m{d^2}(s,h),\forall h \in {\cal H}$, implies that the modulus of continuity $w$ can be taken as 
\begin{align}
w(\varepsilon ) = {m^{ - 1/2}}\varepsilon . \nonumber 
\end{align}
To evaluate $\phi$, we first introduce own way to measuring the “size” of the class ${\cal A}$: the random combinatorial entropy. This entropy is defined as ${{\rm H}_{\cal A}} = \log \# \{ A \cap \{ {X_1}, \ldots ,{X_N}\} ,A \in {\cal A}\} $, which is related to the VC-dimension $d_{vc}$ of ${\cal A}$ via Sauer’s lemma (see~\citep{b19}) which ensures that ${{\rm H}_{\cal A}} \le {d_{vc}}(1 + \log (N/{d_{vc}}))$ whenever $N \ge d_{vc}$. Then, we have 
\begin{align}
\phi (\sigma ) = K\sigma \sqrt {(1 \vee {\mathbb E}[{{\rm H}_{\cal A}}])}, \nonumber
\end{align}
or 
\begin{align}
\phi (\sigma ) = K\sigma \sqrt {{d_{vc}}(1 + \log ({\sigma ^{ - 1}} \vee 1))} , \nonumber 
\end{align}
for an absolute constant $K$. 
In both case above, the assumption~\eqref{applem2eq2} is satisfied and we can apply Lemma~\ref{applem2} with $w \equiv 1$ or $w$ define by the above. 
When $\phi$ is given by the first equation above, the solution ${\varepsilon _ * }$ of Equation~\eqref{applem2eq3} can be explicitly computed by $w \equiv 1$ or $w(\varepsilon ) = {m^{ - 1/2}}\varepsilon $. Hence, the result of Lemma~\ref{applem2} holds with 
\begin{align}
\varepsilon _ * ^2 = \left( {\frac{{{K^2}(1 \vee {\mathbb E}[{{\rm H}_{\cal A}}])}}{{mN}}} \right) \wedge \sqrt {\frac{{{K^2}(1 \vee {\mathbb E}[{{\rm H}_{\cal A}}])}}{N}} . \nonumber 
\end{align}
In the second case of $\phi$, $w \equiv 1$ implies by Equation~\eqref{applem2eq3} that $ \varepsilon _ * ^2 = K\sqrt {{d_{vc}}/N} $, while if $w(\varepsilon ) = {m^{ - 1/2}}\varepsilon $, then
\begin{align}
\varepsilon _ * ^2 = K{\varepsilon _ * }\sqrt {\frac{{{d_{vc}}}}{N}} \sqrt {1 + \log ((\sqrt m {\varepsilon _ * }) \vee 1)} . \nonumber
\end{align}
Since ${1 + \log ((\sqrt m {\varepsilon _ * }) \vee 1)} \ge 1$ and $K \ge 1$, we derive from the above result that 
\begin{align}
\varepsilon _ * ^2 \ge \frac{{{d_{vc}}}}{{mN}}. \nonumber 
\end{align}
Plugging this inequality in the logarithmic factor of the result in the second case of $\phi$ yields 
\begin{align}
\varepsilon _ * ^2 \le K{\varepsilon _ * }\sqrt {\frac{{{d_{vc}}}}{{mN}}} \sqrt {1 + \log ((N{m^2}/{d_{vc}}) \vee 1)}. \nonumber 
\end{align}
Hence, 
\begin{align}
\varepsilon _ * ^2 \le {K^2}\frac{{{d_{vc}}(1 + \log ((N{m^2}/{d_{vc}}) \vee 1))}}{{mN}}. \nonumber 
\end{align}
Thus, the result of Lemma~\ref{applem2} holds with
\begin{align}
\varepsilon _ * ^2 = {K^2}\left[ {\frac{{{d_{vc}}(1 + \log ((N{m^2}/{d_{vc}}) \vee 1))}}{{mN}} \wedge \sqrt {\frac{{{d_{vc}}}}{N}} } \right]. \nonumber
\end{align}
Then, we complete the proof. 
\end{proof}

\noindent\textbf{Remarks}. There is a basic assumption for Lemma~\ref{applem2} and Lemma~\ref{applem3}, which is that the data are i.i.d. Namely, these lemmas hold within a single domain in DG.

\begin{lemma}[Expectation Bound via Jensen-Shannon divergence]
\label{applem4}
Let $P$ and $Q$ be two probability distributions. Let $f: {\cal Z} \to {\mathbb R}$ be a measurable function satisfying 
\begin{align}
\left\|f\right\|_{\infty}:=\mathop{\sup}\limits_{Z \in{\cal Z}}{\left| f(Z) \right|}\le M \le \infty. \nonumber 
\end{align}
Then,
\begin{align}
\left| {\mathbb E}_P[f]-{\mathbb E}_Q[f]\right|\le 2M\sqrt{2JS(P||Q)}, \nonumber 
\end{align}
or 
\begin{align}
\left| {\mathbb E}_P[f]-{\mathbb E}_Q[f]\right|\le 2\sqrt{2}M d_{JS}(P,Q). \nonumber 
\end{align}
\end{lemma}
\begin{proof}
The result follows from standard inequalities in information theory. First, by a standard total variation bound, 
\begin{align}
\left| {\mathbb E}_P[f]-{\mathbb E}_Q[f]\right|\le 2\left\|f\right\|_{\infty}d_{TV}(P,Q). \nonumber 
\end{align}
Applying the Pinsker's inequality, we have 
\begin{align}
\left(\frac{1}{2}d_{TV}(P,Q)\right)^2 &\le \frac{1}{2}KL(P||M) \nonumber \\
\left(\frac{1}{2}d_{TV}(P,Q)\right)^2 &\le \frac{1}{2}KL(Q||M). \nonumber 
\end{align}
Here, $M=(P+Q)/2$, $d_{TV}(P,M)=(1/2)d_{TV}(P,Q)$, and $d_{TV}(Q,M)=(1/2)d_{TV}(P,Q)$. Then, according to $JS(P;Q)=1/2KL(P;M)+1/2KL(Q;M)$, we obtain 
\begin{align}
2\left(\frac{1}{2}d_{TV}(P,Q)\right)^2 \le \frac{1}{2}JS(P||Q). \nonumber 
\end{align}

Here, we clarify how this lemma can be used to derive some results in our paper. Let $Z=(h(X),Y)$ and consider $\gamma$ as $f$. Then, assuming $\|\gamma\|_{\infty} \le M$, we have 
\begin{align}
{\mathbb E}_{Z \sim P}[|\gamma(Z)|]-{\mathbb E}_{Z \sim Q}[|\gamma(Z)|] &\le 2M \sqrt{2JS(P_{(h(X),Y)}||Q_{(h(X),Y)})}, \nonumber \\
&\le 2M \sqrt{2JS(P_{(X,Y)}||Q_{(X,Y)})}, \nonumber 
\end{align}
considering $|a| - |b| \le |a - b |$. The last inequality is a consequence of the Data Processing Inequality from information theory, and $h$ is the same for both $P$ and $Q$. If this condition does not hold, i.e., $h_1$ for $P$ and $h_2$ for $Q$, we obtain 
\begin{align}
{\mathbb E}_{Z \sim P}[|\gamma(Z)|]-{\mathbb E}_{Z \sim Q}[|\gamma(Z)|] &\le 2M \sqrt{2JS(P_{(h_1(X),Y)}||Q_{(h_2(X),Y)})}. \nonumber 
\end{align} 

\end{proof}

\bibliography{sample}

\end{document}